\documentclass{article}

\usepackage{microtype}
\usepackage{graphicx}
\usepackage{subcaption}
\usepackage{booktabs} % for professional tables
\usepackage{subcaption}
\usepackage{cuted}

\usepackage{hyperref}

\usepackage[preprint]{icml2026}

\usepackage{amsmath}
\usepackage{amssymb}
\usepackage{mathtools}
\usepackage{amsthm}
\usepackage[most]{tcolorbox}

\usepackage{booktabs}
\usepackage{multirow}
\usepackage{siunitx}

\usepackage{xcolor}
\usepackage{colortbl}
\usepackage{balance}

\usepackage[capitalize,noabbrev]{cleveref}

\definecolor{myblue}{HTML}{E1F4FD}

\theoremstyle{plain}
\newtheorem{theorem}{Theorem}[section]

\theoremstyle{definition}
\newtheorem{definition}[theorem]{Definition}

\theoremstyle{remark}

\usepackage[textsize=tiny]{todonotes}

\icmltitlerunning{CoPE: Clipped RoPE as A Scalable Free Lunch for Long Context LLMs}

\begin{document}

\twocolumn[
  \icmltitle{CoPE: Clipped RoPE as A Scalable Free Lunch\\for Long Context LLMs}



  \icmlsetsymbol{equal}{*}

  \begin{icmlauthorlist}
    \icmlauthor{Haoran Li}{cmu,jhu}
    \icmlauthor{Sucheng Ren}{jhu}
    \icmlauthor{Alan Yuille}{jhu}
    \icmlauthor{Feng Wang}{jhu}
  \end{icmlauthorlist}

  \icmlaffiliation{cmu}{Carnegie Mellon University}
  \icmlaffiliation{jhu}{Johns Hopkins University}
  
  \icmlcorrespondingauthor{Haoran Li}{haoranl4@cs.cmu.edu}
  \icmlcorrespondingauthor{Feng Wang}{fwang60@jh.edu}

  \icmlkeywords{Machine Learning, ICML}
  
  \vskip 0.3in
]

% \begin{strip}
% \begin{center}
%     \centering
%     \includegraphics[width=0.95\linewidth]{figs/Fig1_1.pdf}
%     \captionof{figure}{\textbf{CoPE consistently outperforms RoPE} across diverse tasks, context lengths, and long context training strategies: \textbf{(a)} Qwen-ABF, which increases the base frequency of RoPE using the ABF technique \cite{xiong-etal-2024-effective}; \textbf{(b)} LLaMA-ABF, which only increases the base frequency of low-frequency components in RoPE; \textbf{(c)} DeepSeek-YaRN, which employs YaRN \cite{peng2024yarn} for context scaling.}
%     \label{fig:1}
% \end{center}
% \vspace{7pt}
% \end{strip}



\printAffiliationsAndNotice{}  % no special notice (required even if empty)

\begin{abstract}
Rotary Positional Embedding (RoPE) is a key component of context scaling in Large Language Models (LLMs). While various methods have been proposed to adapt RoPE to longer contexts, their guiding principles generally fall into two categories: (1) \textit{out-of-distribution (OOD) mitigation}, which scales RoPE frequencies to accommodate unseen positions, and (2) \textit{Semantic Modeling}, which posits that the attention scores computed with RoPE should always prioritize semantically similar tokens. In this work, we unify these seemingly distinct objectives through a minimalist intervention, namely CoPE: soft clipping low-frequency components of RoPE. CoPE not only eliminates OOD outliers and refines semantic signals, but also prevents spectral leakage caused by hard clipping. Extensive experiments demonstrate that simply applying our soft clipping strategy to RoPE yields significant performance gains that scale up to 256k context length, validating our theoretical analysis and establishing CoPE as a new state-of-the-art for length generalization. Our code, data, and models are available at \url{https://github.com/hrlics/CoPE}.
\end{abstract}
% Abstracts must be a single paragraph, ideally between 4--6 sentences long.
% Gross violations will trigger corrections at the camera-ready phase.
% adjusts the base frequency of RoPE

% which simultaneously eliminates OOD outliers and refines semantic signals. 

\section{Introduction}
% 1. RoPE is important to length generalization in LLMs, explain how this is done -> basically extend the first sentence in abstract: "Rotary Positional Embedding (RoPE) is a key component of context scaling in Large Language Models (LLMs)." long context -> the central technique is to modify rope and how is it done?

Long context Large Language Models (LLMs) have become a cornerstone of critical domains such as coding agents \cite{jimenez2024swebench, Claudecode}, agentic memory \cite{yu2025memagent, mem0}, and long-horizon reasoning \cite{qiao2025webresearcher, zhou2025mem1, sinha2025illusion}. To achieve context scaling, a long context training stage is often required after initial pre-training, where the frequencies within Rotary Positional Embedding (RoPE) \cite{su2024roformer} are modified to fit the target context length, followed by continued training on long sequences.

% 2. discuss previous works on modifying rope for length generaliztion, explaion what is OOD mitigation and semantic modeling -> talk about the 2 guiding principles for modifying RoPE. The second sentence in abstract

While existing works have proposed various methods to adapt RoPE to longer contexts, their guiding principles generally fall into two categories: \textbf{\textit{OOD mitigation}} and \textbf{\textit{semantic modeling}}. Specifically, RoPE divides the query and key vectors into two-dimensional chunks, and rotates each chunk at a specific frequency. For low-frequency components that do not complete a full rotation during pre-training, extrapolating to unseen positions leads to severe OOD issues. Therefore, several \textbf{\textit{OOD mitigation}} strategies, including Position Interpolation (PI) \cite{chen2023extending}, NTK \cite{NTK}, YaRN \cite{peng2024yarn}, and LongRoPE \cite{ding2024longrope, shang2025longrope}, are introduced to scale the frequencies so that extended contexts are mapped back to the original position range. In contrast, another line of research is inspired by \textbf{\textit{semantic modeling}}, which posits that the attention scores computed with RoPE should always prioritize semantically similar tokens. Men et al. \yrcite{men2024base} show that the rotation matrix in attention would degrade the model’s ability to discriminate relevant tokens from irrelevant ones as the relative distance increases, motivating the use of a higher base frequency. The ABF technique \cite{xiong-etal-2024-effective} arrives at the same strategy, claiming that increasing the base frequency mitigates the long-term decay in RoPE and improves long context modeling.

\begin{figure}
    \centering
    \includegraphics[width=0.97\linewidth]{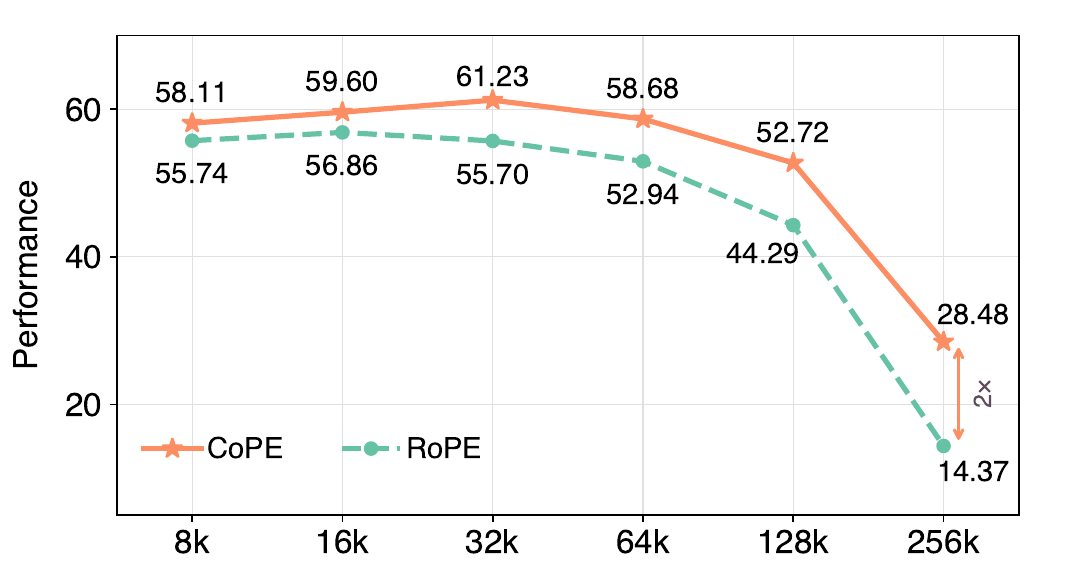}
    \caption{\textbf{Performance comparison between CoPE and RoPE.} With a simple \textit{soft} clipping strategy, CoPE effectively improves RoPE's performance both within the training range and during extrapolation. The training context length here is $64$k.}
    \label{fig:1}
\end{figure}

% 3. Despite,...  talk how we unifiy these through simple clipping -> the 3rd sentence in our abs "In this work, we theoretically unify these objectives through a minimalist intervention (CoPE): soft clipping lowfrequency components of RoPE, which simultaneously eliminates OOD outliers and refines semantic signals."

Despite their improved performance, these two lines of research are typically treated as tackling distinct aspects of long context modeling. However, we argue that they stem from the same underlying issue: \textit{\textbf{the suboptimal behavior of low-frequency components in the extrapolation regime.}} Through theoretical analysis of RoPE’s frequency spectrum, we show that low-frequency components simultaneously govern OOD behavior under extrapolation and the stability of semantic attention over long contexts. Motivated by this insight, we propose CoPE, a minimalist intervention that softly clips the low-frequency components of RoPE. This simple and effective strategy not only suppresses OOD outliers and refines semantic signals, but also prevents spectral leakage caused by hard clipping, providing a plug-and-play solution that can be seamlessly integrated into existing LLMs for better long context capability.

% 4. talk about the compatibility, experimental result on qwen-abf, llama-abf, and deepseek yarn

To validate the effectiveness and compatibility of CoPE, we conduct extensive experiments that align with the long context recipe used in Qwen3 \cite{yang2025qwen3}, i.e., employing the ABF technique \cite{xiong-etal-2024-effective} during long-context training and YaRN for test-time extrapolation. By simply replacing the standard RoPE with CoPE while keeping all other configurations unchanged, we observe consistent and significant improvements across diverse tasks and context lengths, as shown in Figure~\ref{fig:1}. Notably, at context lengths up to 256k tokens, CoPE achieves nearly \textbf{2$\times$} the performance of RoPE, while also maintaining superior performance within the training range. Together, our theoretical analysis and empirical results establish CoPE as a simple, general, and highly effective drop-in replacement for RoPE in long context LLMs.

% 5. summarize the contribution
Our main contributions can be summarized as follows:
\vspace{-8pt}
\begin{itemize}
    \setlength{\itemsep}{1pt}
    \item We provide a unified perspective on long context adaptations of RoPE, showing that both OOD mitigation and semantic modeling methods originate from the suboptimal behavior of low-frequency components in the extrapolation regime.
    \item Based on this insight, we propose \textbf{CoPE}, a minimalist and principled modification to RoPE that softly attenuates low-frequency components, eliminating OOD outliers, refining semantic signals, and preventing the spectral leakage induced by hard clipping.
    \item We conduct extensive experiments to demonstrate that CoPE is a \emph{simple} and \emph{scalable} drop-in replacement for RoPE, consistently improving performance across diverse tasks and context lengths up to $256$k.
\end{itemize}

\section{Preliminaries}
\textbf{Rotary Position Embedding (RoPE).} Transformer-based models \cite{vaswani2017attention} rely on Positional Encodings (PEs) to explicitly incorporate sequential information. Among various PEs, Rotary Position Embedding (RoPE) \cite{su2024roformer} has become the dominant choice in modern LLMs. Let $\mathbf{x}_i \in \mathbb{R}^{d}$ denote the $d$-dimensional token embedding of the $i$-th token in a sequence. Consider the $n$-th query vector $\mathbf{q}_n$ and the $m$-th key vector $\mathbf{k}_m$, RoPE partitions the dimensions into $d/2$ chunks, e.g., $\mathbf{q}_n=[\mathbf{q}_n^{(0)}; \mathbf{q}_n^{(1)};\dots;\mathbf{q}_n^{(d/2-1)}]$. Each chunk is assigned a unique rotation frequency $\theta_i=b^{-2i/d}, i \in \{0,1,\dots,d/2-1\}$, where $b$ is a pre-defined base frequency (typically set to $10,000$). The rotation is achieved through a rotation matrix $\mathbf{R}_{n} \in \mathbb{R}^{d \times d}$, which can be formulated as follows:
\begin{equation}
    \scriptsize
    \begin{pmatrix}
        \cos(n\theta_0) & -\sin(n\theta_0)  & \cdots & 0 & 0 \\
        \sin(n\theta_0)  &  \cos(n\theta_0)  & \cdots & 0 & 0 \\
        \vdots & \vdots & \ddots & \vdots & \vdots \\
        0 & 0 & \cdots & \cos(n\theta_{d/2-1})  & -\sin(n\theta_{d/2-1}) \\
        0 & 0 & \cdots & \sin(n\theta_{d/2-1}) &  \cos(n\theta_{d/2-1})
    \end{pmatrix}.
    \label{eq:RoPE}
\end{equation}
With this block-diagonal rotation matrix, the attention score\footnote{Here, we omit the softmax function and \(1/\sqrt{d}\) scaling in standard Transformer \citep{vaswani2017attention} for simplicity.} between $\mathbf{q}_n$ and $\mathbf{k}_m$ is computed as:
\begin{equation}
  A_{n,m}=(\mathbf{R}_n\mathbf{q}_n)^{\top}(\mathbf{R}_m\mathbf{k}_m)=\mathbf{q}_n^{\top}\mathbf{R}_{m-n}\mathbf{k}_m,
\label{eq:attention}
\end{equation}
where $(m-n)$ is the relative distance between $\mathbf{q}_n$ and $\mathbf{k}_m$.

\section{Analysis}
\label{sec:analysis}
% Long Context Adaptions of RoPE}
In this section, we conduct a comprehensive theoretical analysis of existing methods that adapt RoPE to longer contexts. We begin by highlighting the underlying guiding principles of prior methods, namely \textbf{\textit{OOD mitigation}} and \textbf{\textit{semantic modeling}}. We then show that these two seemingly distinct objectives both originate from the same root cause: the suboptimal behavior of low-frequency components in the extrapolation regime.

\begin{figure*}[t]
    \vspace{6pt}
    \centering
    \begin{subfigure}[b]{0.55\textwidth}
        \centering
        \includegraphics[width=\textwidth]{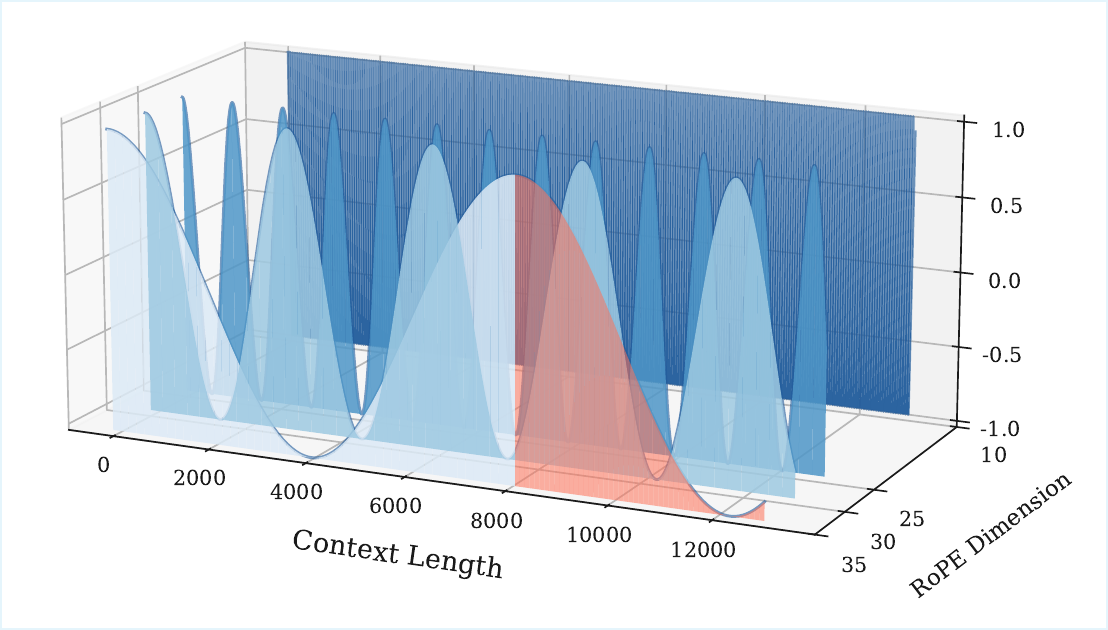}
        \caption{RoPE Frequencies and OOD Issue.}
    \end{subfigure}
    % \hfill
    \begin{subfigure}[b]{0.4\textwidth}
        \centering
        \includegraphics[width=0.9\textwidth]{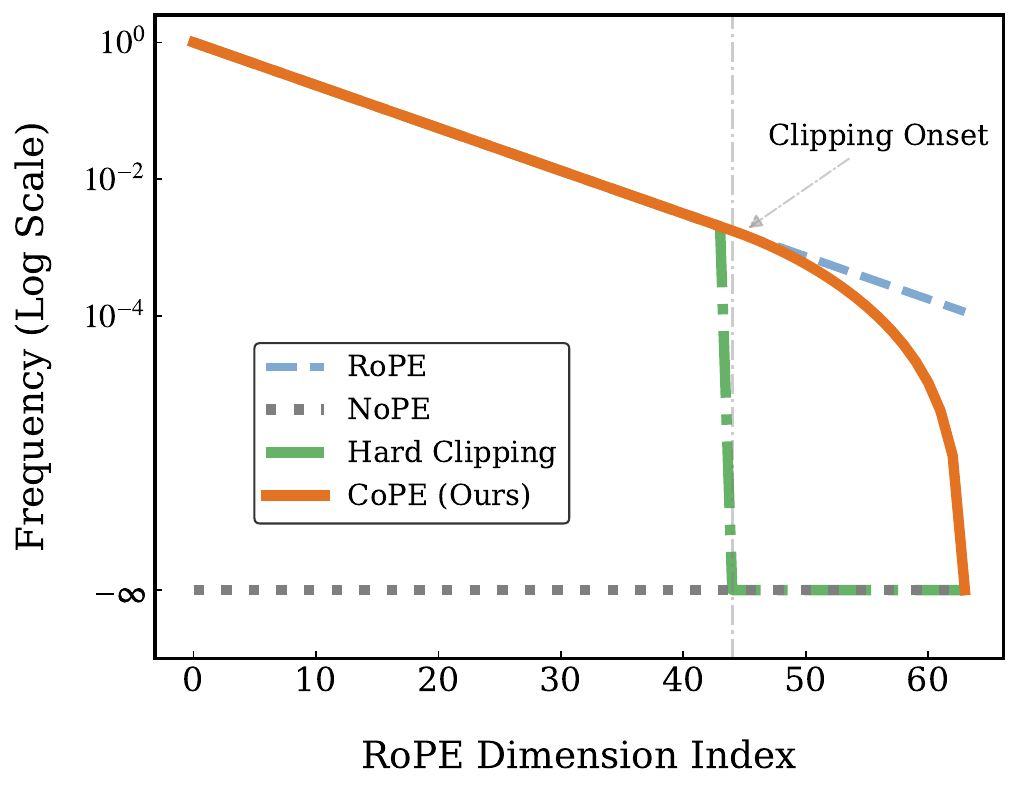}
        \vspace{3pt}
        \caption{Spectral Comparison.}
        \label{fig:2_b}
    \end{subfigure}
    \vspace{7pt}
    \caption{\textbf{(a) Visualization of RoPE frequencies.} Low-frequency components in higher dimensions possess longer periods. The region shaded in red marks where the period exceeds the pre-training context window, leading to OOD extrapolation. \textbf{(b) Spectral comparison.} Unlike RoPE which keeps unstable low frequencies (blue), or Hard Clipping which causes an abrupt cut-off and spectral leakage, CoPE implements a soft decay strategy starting from the clipping onset, simultaneously eliminating OOD outliers and refines semantic signals.}
    \label{fig:2}
\end{figure*}

\subsection{RoPE OOD Theory}
\label{sec:rope_ood}

\textbf{Background.} Recall that RoPE divides the query and key vectors into $2$-dimensional chunks and rotates each chunk at a frequency of $\theta_{i}=b^{-2i/d}, i\in\{0,1,\dots,d/2-1\}$, where $b$ is the base frequency and is usually set to $10,000$. Given the periodicity of sinusoidal functions, we know that for each chunk with frequency $\theta_i$, the corresponding period can be calculated as follows:
\begin{equation}
T_i = \frac{2\pi}{\theta_i}.
\end{equation}
Since $\theta_i$ decreases as the dimensional index $i$ increases, the low-frequency components in higher dimensions possess longer periods, potentially exceeding the pre-training context window. For example, the pre-training context window of Llama-3-8B \cite{grattafiori2024llama} is $8192$, while the period of the $35$-th chunk already slightly exceeds this length. Consequently, out of the $64$ chunks, the last $29$ low-frequency chunks fail to experience a single complete period during the pre-training stage, leading to severe OOD issues during extrapolation. In contrast, high-frequency components in lower dimensions complete multiple cycles during pre-training and remain well-behaved even in extrapolation.

\textbf{Critical Dimension in Extrapolation.} Based on the above spectrum analysis and prior work \cite{liu2024scaling, shang2025longrope}, we formally define the critical dimension in RoPE-based extrapolation as follows:

\begin{tcolorbox}[
  colback=gray!7!white,  
  colframe=black,  
  arc=0.8mm,                 
  boxrule=1pt,         
  left=2mm, right=2mm, top=1mm, bottom=1mm,
  enhanced,
  width=\linewidth,
]
\begin{definition}[Critical Dimensions in Extrapolation]
For LLMs pre-trained with context window $L_{pre}$, attention head dimension $d$, and base frequency $b$, only the first $d_{ct}$ dimensions perceive complete periodic patterns during pre-training:
\begin{equation}
    d_{ct} = 2\lceil\frac{d}{2}\log_{b}\frac{L_{pre}}{2\pi} \rceil.
\end{equation}
\vspace{-6pt}
\label{def:critical_dim}
\end{definition}
\end{tcolorbox}

As shown in Figure~\ref{fig:2}, for Llama-3-8B with $L_{pre}=8192$, $d=128$, $b=500,000$, the critical dimension is $70$, which corresponds to the $35$-th rotation chunk as discussed earlier.

\textbf{OOD Mitigation Methods.} To mitigate the OOD behavior of RoPE beyond the critical dimension, several methods have been proposed to scale the frequencies $\theta_i$ so that extended contexts are mapped back to the original position range. For ease of notation, we denote the target context length as $L_t$ and the scaling factor for each frequency $\theta_i$ as $s_i$. Given the scaling factor, the scaled frequency can be calculated as:
\begin{equation}
    \theta^{\prime}_{i} = \frac{\theta_i}{s_i} = \frac{1}{s_i \times b^{2i/d}}.
\end{equation}
Representative works include PI \cite{chen2023extending}, NTK \cite{NTK}, YaRN \cite{peng2024yarn}, and LongRoPE \cite{ding2024longrope, shang2025longrope}. \textit{PI} applies a uniform scaling factor across all RoPE frequencies, i.e., $s_i=\frac{L_t}{L_{pre}}$. While easy to implement, this approach equally stretches all dimensions without considering the distinct behaviors of high- and low-frequency components of RoPE during extrapolation. As a result, it compresses high-frequency components, leading to a loss of local positional resolution. Inspired by the Neural Tangent Kernel (NTK) theory \cite{tancik2020fourier}, which states that neural networks have difficulties learning high-frequency features, \textit{NTK} proposes to scale high frequencies less and low frequencies more with the scaling factor $s_i = (\frac{L_t}{L_{pre}})^{2i/(d-2)}$, effectively alleviating the loss of high-frequency information. Building on NTK, \textit{YaRN} further partitions the frequencies into three groups and applies the following strategy: no scaling for high-frequency components ($s_i=1$), PI-style scaling for low-frequency components ($s_i=\frac{L_t}{L_{pre}}$), and linear interpolation between $1$ and $\frac{L_t}{L_{pre}}$ for intermediate frequencies. \textit{LongRoPE} adopts a perplexity-guided search-based method to estimate the optimal scaling factor $s_i$ for each frequency.

\begin{tcolorbox}[
  colback=blue!5!white,   
  colframe=black,  
  arc=0.8mm,                 
  boxrule=1pt,         
  left=2mm, right=2mm, top=1mm, bottom=1mm,
  enhanced,
  width=\linewidth,
]
\textbf{Takeaway 1.} OOD mitigation methods address extrapolation by interpolating low-frequency components, while minimally perturbing high frequencies. The primary distinction among these methods lies in their choice of per-frequency scaling factors.
\end{tcolorbox}

\begin{figure}[t]
    \centering
    \includegraphics[width=0.8\linewidth]{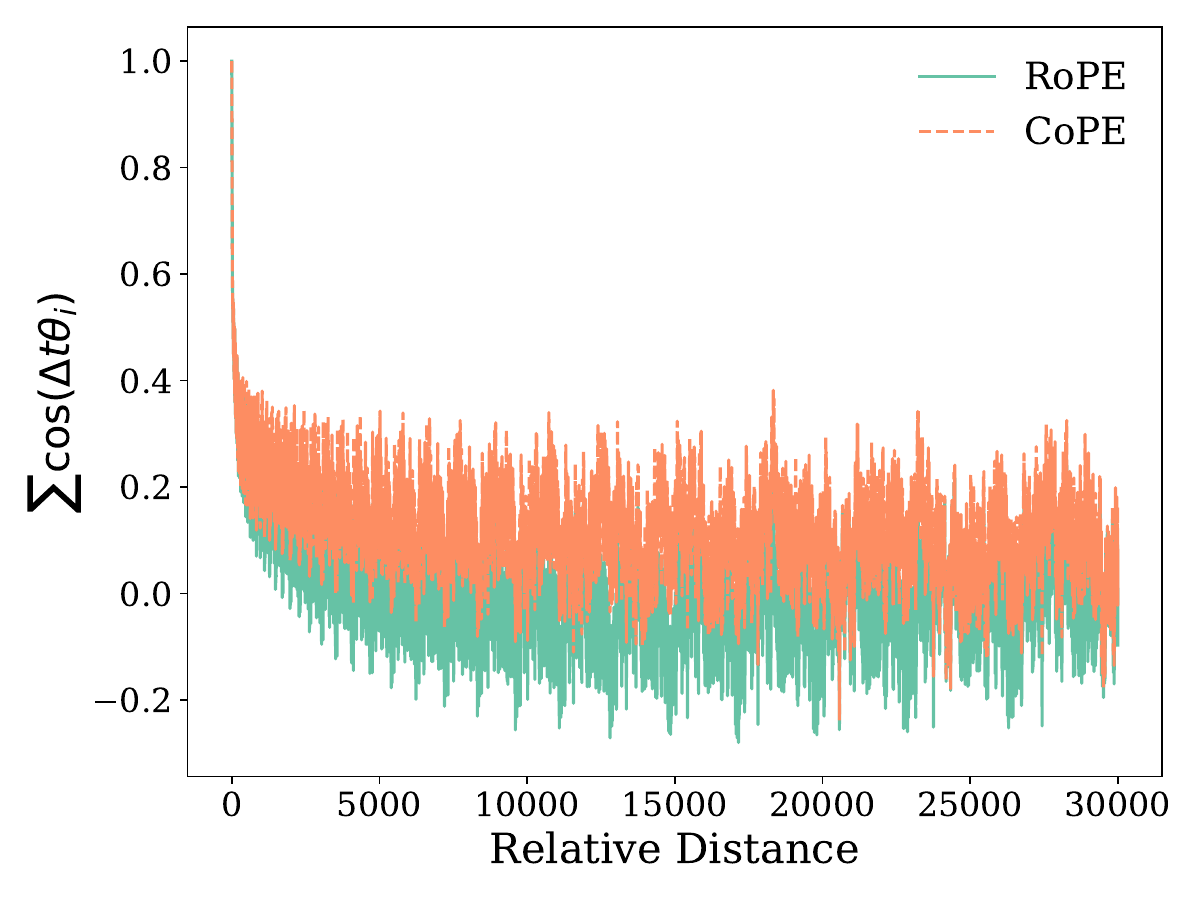}
    \caption{\textbf{Long-term decay of semantic attention.} As relative distance increases, the model’s ability to prefer semantically similar tokens over random ones diminishes. Applying soft clipping to the low-frequency components (CoPE) effectively alleviates this decay, preserving semantic information over long contexts.}
    % \vspace{-2em}
    \label{fig:semantic_decay}
\end{figure}

\subsection{RoPE Semantic Modeling}
\label{sec:rope_semantic}

\textbf{Background.} When RoPE was originally proposed \cite{su2024roformer}, it introduced an important inductive bias known as \textit{\textbf{long-term decay}}: the upper bound of the attention score between two tokens decreases as their relative distance increases. This property encourages each token to attend more to its neighbors. However, Men et al. \yrcite{men2024base} observe that an undesirable decay property also exists: the ability to attend more to semantically similar tokens than random tokens also decays as the relative distance increases. Following Men et al. \yrcite{men2024base}, we denote this property as \textit{\textbf{long-term decay of semantic attention}} and formalize it as follows:
\begin{tcolorbox}[
  colback=gray!7!white,  
  colframe=black,  
  arc=0.8mm,                 
  boxrule=1pt,         
  left=2mm, right=2mm, top=1mm, bottom=1mm,
  enhanced,
  width=\linewidth,
]
\begin{theorem}[Long-term Decay of Semantic Attention]
Assume query $\mathbf{q} \in \mathbb{R}^d$ and key $\mathbf{k} \in \mathbb{R}^d$ have distance $\Delta t$ and \text{i.i.d.} components with standard deviation $\sigma$. Let $\mathbf{k}^{\prime}=\mathbf{q}+\mathbf{\epsilon}$ denote a similar key to the query, where $\mathbf{\epsilon}$ is a zero-mean perturbation. Then, we have:
\begin{equation}
    \mathbb{E}_{\mathbf{q},\mathbf{k}, \epsilon}[\mathbf{q}^{\top}\mathbf{R}_{\Delta t}\mathbf{k}^{\prime}-\mathbf{q}^{\top}\mathbf{R}_{\Delta t}\mathbf{k}] = {2\sigma^2} \sum_{i=0}^{d/2-1} \cos(\Delta t \theta_i),
\end{equation}
\label{eq:semantic_attn}
where $\sum_{i=0}^{d/2-1} \cos(\Delta t \theta_i)$ decays as $\Delta t$ increases.
\label{theorem:semantic_attention}
\end{theorem}
\end{tcolorbox}

The proof is provided in Appendix~\ref{proof:semantic_decay}. Note that the term $\sum_{i=0}^{d/2-1} \cos(\Delta t \theta_i)$ should ideally be greater than zero to ensure more attention is paid to similar tokens than random ones. However, this term does decrease as $\Delta t$ increases, as shown in Figure~\ref{fig:semantic_decay}. Given this observation, Men et al. \yrcite{men2024base} propose to use a higher base frequency $b$, which in turn decreases $\theta_i=b^{-2i/d}$ and alleviates this undesirable decay. Similarly, the ABF technique \cite{xiong-etal-2024-effective} arrives at the same higher base frequency strategy, claiming that increasing the base frequency reduces the general long-term decay of RoPE and improves long context modeling. Given its simplicity and effectiveness, the higher base frequency strategy has been widely adopted in long context training \cite{grattafiori2024llama, yang2025qwen3}. More recently, several work has analyzed how different RoPE frequencies influence attention patterns, concluding that low-frequency components primarily carry semantic information \cite{barbero2025round, jin2025massive}, as they are the most invariant to token relative distance.
\begin{tcolorbox}[
  colback=blue!5!white,   
  colframe=black,  
  arc=0.8mm,                 
  boxrule=1pt,         
  left=2mm, right=2mm, top=1mm, bottom=1mm,
  enhanced,
  width=\linewidth,
]
\textbf{Takeaway 2.} RoPE secretly induces a long-term decay of semantic attention that is primarily governed by the low-frequency components, revealing them as an unreliable semantic channel. 
\end{tcolorbox}

\subsection{All Roads Lead to Low-Frequency Components}

Our analysis above reveals a unifying insight: both \textit{OOD extrapolation} and \textit{long-term decay of semantic attention} stem from the same root cause: \textbf{\textit{the suboptimal behavior
of low-frequency components in the extrapolation regime.}} Specifically, from the OOD perspective, low-frequency components possess periods exceeding the pre-training context window, resulting in OOD extrapolation. Meanwhile, from the semantic modeling perspective, low frequencies serve as the semantic channel that distinguishes similar tokens from random ones, yet this ability decays as context length increases. Our unified perspective suggests a simple yet effective design principle: stabilizing the behavior of low-frequency components is sufficient to mitigate OOD extrapolation and preserve long-range semantic attention.

\begin{tcolorbox}[
  colback=blue!5!white,   
  colframe=black,  
  arc=0.8mm,                 
  boxrule=1pt,         
  left=2mm, right=2mm, top=1mm, bottom=1mm,
  enhanced,
  width=\linewidth,
]
\textbf{Takeaway 3.} 
OOD extrapolation and long-term semantic decay are two manifestations of the \textbf{same} underlying issue: the suboptimal behavior of low-frequency components beyond the pre-training regime.
\end{tcolorbox}

\section{CoPE: Clipped Rotary Position Embedding}
Motivated by our analysis in Section~\ref{sec:analysis}, we propose Clipped Rotary Position Embedding (\textbf{CoPE}), a simple yet effective method that softly clips the low-frequency components of RoPE, as illustrated in Figure~\ref{fig:2_b}. CoPE not only eliminates OOD outliers and refines semantic signals, but also prevents severe spectrum leakage induced by hard clipping, thereby scaling favorably with increased context window.

\subsection{Spectral Analysis}

To stabilize low-frequency components, a straightforward approach is to directly set them to zero, i.e., hard clipping. For example, Babero et al. \yrcite{barbero2025round} identify the low frequencies as the semantic channel and propose to stabilize them by clipping the lowest $25\%$ or $75\%$ frequencies, resulting in lower validation perplexity on a 2B-scale model with 8k context length. However, hard clipping introduces an abrupt spectral cutoff, which can distort the remaining frequency components and undermine the stability of positional information, particularly in long-context scenarios. To elaborate, we first reframe the attention mechanism with RoPE through the lens of Non-Uniform Discrete Fourier Transform (NUDFT). As shown in Equation~\ref{eq:attention}, the dot-product attention between the $n$-th query vector $\mathbf{q}_n$ and the $m$-th key vector $\mathbf{k}_n$ is calculated as:

\begin{figure}[t]
    \centering
    \includegraphics[width=0.8\linewidth]{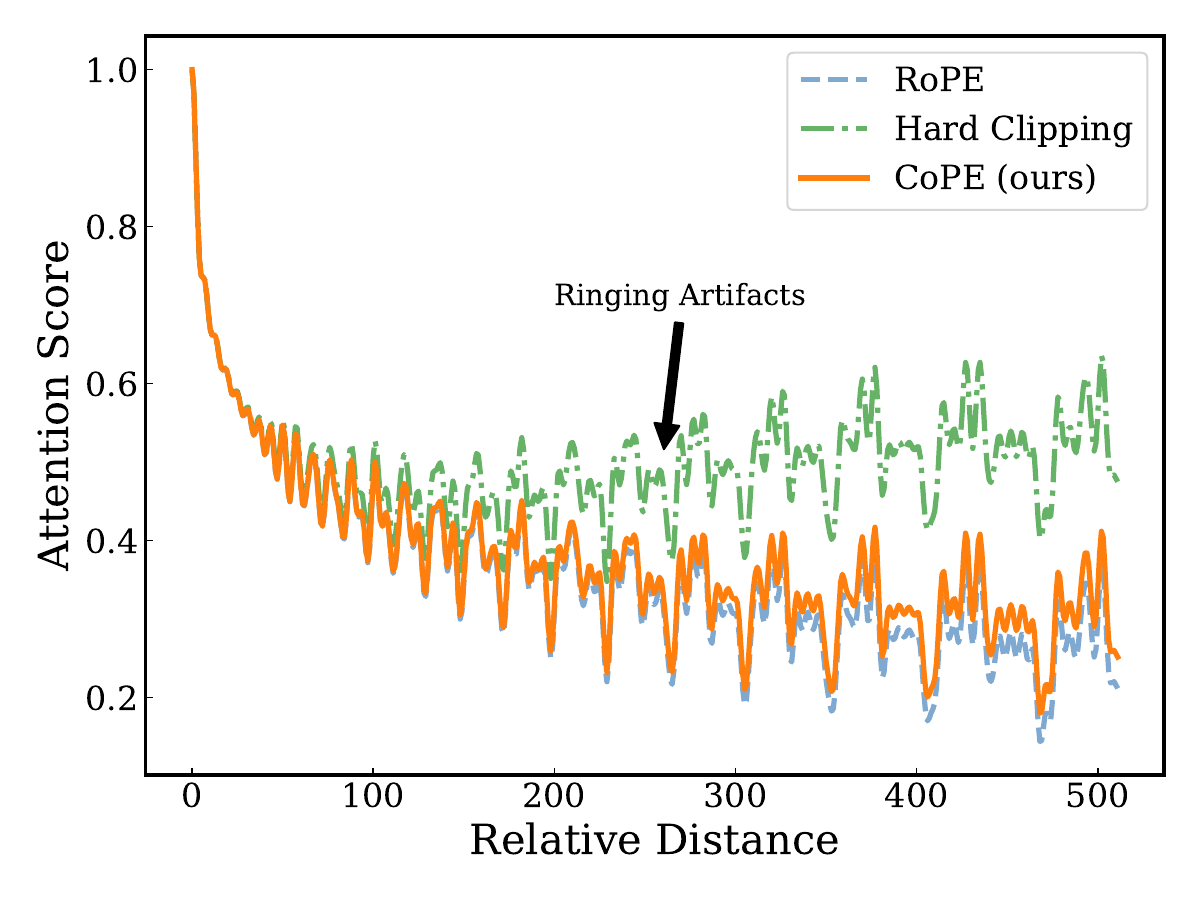}
    \caption{\textbf{Ringing artifacts caused by hard clipping.} Directly applying a hard clipping to the low-frequency components introduces an abrupt spectral cutoff, which causes spectral leakage and manifests as long-range oscillatory ringing in the attention signal (Gibbs phenomenon).}
    \label{fig:ring}
\end{figure}

\begin{equation}
  A_{n,m}=(\mathbf{R}_n\mathbf{q}_n)^{\top}(\mathbf{R}_m\mathbf{k}_m)=\mathbf{q}_n^{\top}\mathbf{R}_{m-n}\mathbf{k}_m,
\end{equation}

which can be further transformed into

\begin{equation}
    \small
    A(\tau) = \text{Re}\left[ \sum_{j=0}^{d/2-1} (\mathbf{q}^{(j)}_{n} \mathbf{k}^{(j)*}_{m}) e^{i \theta_j \tau} \right] = \sum_{j=0}^{d/2-1} A_j \cos(\theta_j \tau),
\end{equation}

where $\tau = m-n$ denotes the relative distance. This formulation reveals that the attention score computed with RoPE achieves an inverse NUDFT with frequency components $\theta_j = b^{-2j/d}, j \in [0,d/2)$. Now, we analyze the impact of hard clipping using a continuous approximation of $A(\tau)$ in the large-$d$ limit, which provides clearer theoretical insight.

\begin{tcolorbox}[
  colback=gray!7!white,  
  colframe=black,  
  arc=0.8mm,                 
  boxrule=1pt,         
  left=2mm, right=2mm, top=1mm, bottom=1mm,
  enhanced,
  width=\linewidth,
]
\begin{theorem}[Spectral Leakage from Hard Clipping]
Let $A(\tau)$ be the continuous attention score. Hard high-pass filter at cutoff $\theta_c$ yields a $\tilde{A}(\tau) = A(\tau) + E(\tau)$, where the error term is:
\begin{equation}
E(\tau) = - A(\tau) * \left( \frac{\theta_c}{\pi} \text{sinc}\left(\frac{\theta_c \tau}{\pi}\right) \right).
\end{equation}
The slow $O(1/\tau)$ decay of the sinc kernel introduces Gibbs oscillations, disrupting the general decay of $A(\tau)$ and inducing spurious long-range correlations.
\label{theorem:hard_clip}
\end{theorem}
\end{tcolorbox}

The proof is provided in Appendix~\ref{proof:hard_clip}. Theorem~\ref{theorem:hard_clip} shows that the slowly decaying $O(1/\tau)$ envelope of the sinc kernel is a direct consequence of the sharp spectral discontinuity introduced by hard clipping. As a result, the attention scores exhibit Gibbs ringing, where oscillatory artifacts disrupt the general monotonicity of decay and cause spurious long-range correlations, as illustrated in Figure~\ref{fig:ring}.

\subsection{Soft Clipping Strategy}

To address the above challenge, our CoPE introduces a soft clipping strategy, which applies a smooth spectral taper (e.g., a cosine window) to the low frequencies. By Fourier duality, this soft clipping yields a rapidly decaying kernel in the time domain, suppressing unstable low-frequency components without inducing long-range spurious correlations.

Specifically, instead of applying a binary mask $\mathbf{1}_{\theta > \theta_c}$, we assign a scalar weight $w_j \in [0, 1]$ to each frequency component $\theta_j$. To minimize spectral discontinuity, we employ a cosine-decay taper. The weights $w_j$ are defined as a function of the frequency $\theta_j$:

\begin{equation}
\small
w(\theta_j) = 
\begin{cases} 
1, & \theta_j \ge \theta_{\text{start}} \\
\frac{1}{2} \left[ 1 + \cos\left( \pi \frac{\theta_{\text{start}} - \theta_j}{\theta_{\text{start}} - \theta_{\min}} \right) \right], & \theta_{\min} \le \theta_j < \theta_{\text{start}}
\end{cases},
\label{eq:soft_clip_weights} 
\end{equation}

where $\theta_{\text{start}}$ denotes the clipping onset and $\theta_{\min}$ is the lowest frequency. This strategy is highly practical as it allows for seamless integration into modern LLM frameworks. By simply modifying the initialization of the RoPE frequency, CoPE can be applied as a drop-in replacement without altering the model architecture. This ensures full compatibility with optimized inference kernels, such as FlashAttention \cite{dao2023flashattention2}, while maintaining standard inference speeds.

\section{Experiment}
\begin{table*}[t!]
\centering
\renewcommand{\arraystretch}{1.13}
\small
\caption{\textbf{Main results on HELMET benchmark across diverse real-world tasks.} Models are trained with 64k context length and evaluated up to 256k to assess length generalization. CoPE consistently outperforms RoPE and hard clipping, with performance gains scaling favorably with context length. The best results are \textbf{bold}, while ``--'' indicates unavailable benchmark data at that context length.}
\label{tab:main_results}
\begin{tabular}{>{\centering\arraybackslash}p{1.5cm}|>{\raggedright\arraybackslash}p{1.3cm}|*{6}{>{\centering\arraybackslash}p{1.51cm}}}
\toprule[1.1pt]
\multirow{2}{*}{\textbf{Task}} 
& \multirow{2}{*}{\textbf{Method}} 
& \multicolumn{6}{c}{\textbf{Context Length}} \\[4pt]
& & \textbf{8k} & \textbf{16k} & \textbf{32k} & \textbf{64k} & \textbf{128k} & \textbf{256k} \\
\midrule

% -------- Summ --------
\multirow{3}{*}{\textbf{Summ.}}
& RoPE & 29.18 & 28.46 & 21.76 & 11.10 & 6.31 & 9.06 \\
& HardClip & 25.68 & 25.70 & 18.55 & 6.93 & 9.33 & 8.60 \\
& \cellcolor{blue!6}{\textbf{CoPE}} & \cellcolor{blue!6}{\textbf{29.76}} & \cellcolor{blue!6}{\textbf{30.81}} & \cellcolor{blue!6}{\textbf{32.78}} & \cellcolor{blue!6}{\textbf{30.88}} & \cellcolor{blue!6}{\textbf{27.89}} & \cellcolor{blue!6}{\textbf{32.37}} \\
\midrule

% -------- QA --------
\multirow{3}{*}{\textbf{QA}}
& RoPE & 6.46 & 8.39 & 8.52 & 7.67 & 8.21 & 7.93 \\
& HardClip & 7.44 & 9.28 & 10.16 & 10.31 & 9.31 & 9.24 \\
& \cellcolor{blue!6}{\textbf{CoPE}} & \cellcolor{blue!6}{\textbf{13.10}} & \cellcolor{blue!6}{\textbf{16.89}} & \cellcolor{blue!6}{\textbf{21.02}} & \cellcolor{blue!6}{\textbf{15.07}} & \cellcolor{blue!6}{\textbf{18.23}} & \cellcolor{blue!6}{\textbf{19.06}} \\
\midrule

% -------- ICL --------
\multirow{3}{*}{\textbf{ICL}}
& RoPE & 74.60 & 80.20 & 83.40 & 85.50 & 82.10 & -- \\
& HardClip & 73.10 & 77.00 & 79.80 & 82.20 & 77.30 & -- \\
& \cellcolor{blue!6}{\textbf{CoPE}} & \cellcolor{blue!6}{\textbf{79.40}} & \cellcolor{blue!6}{\textbf{83.70}} & \cellcolor{blue!6}{\textbf{85.50}} & \cellcolor{blue!6}{\textbf{86.40}} & \cellcolor{blue!6}{\textbf{84.70}} & \cellcolor{blue!6}{--} \\
\midrule

% -------- Recall --------
\multirow{3}{*}{\textbf{Recall}}
& RoPE & 99.75 & 99.13 & 98.13 & 97.63 & 71.38 & 26.13 \\
& HardClip & 99.75 & 98.50 & 98.50 & 94.38 & 82.13 & 36.86 \\
& \cellcolor{blue!6}{\textbf{CoPE}}& \cellcolor{blue!6}{99.63} & \cellcolor{blue!6}{98.88} & \cellcolor{blue!6}{\textbf{99.00}} & \cellcolor{blue!6}{\textbf{97.88}} & \cellcolor{blue!6}{76.00} & \cellcolor{blue!6}{34.00} \\
\midrule

% -------- RAG --------
\multirow{3}{*}{\textbf{RAG}}
& RoPE & 68.38 & 67.44 & 66.67 & 62.78 & 53.44 & -- \\
& HardClip & 68.06 & 67.50 & 66.11 & 59.72 & 62.05 & -- \\
& \cellcolor{blue!6}{\textbf{CoPE}} & \cellcolor{blue!6}{\textbf{68.67}} & \cellcolor{blue!6}{\textbf{67.72}} & \cellcolor{blue!6}{\textbf{67.83}} & \cellcolor{blue!6}{\textbf{63.17}} & \cellcolor{blue!6}{56.78} & \cellcolor{blue!6}{--} \\
\midrule

% -------- Avg --------
\multirow{3}{*}{\textbf{Average}}
& RoPE & 55.74 & 56.86 & 55.70 & 52.94 & 44.29 & 14.37 \\
& HardClip & 54.81 & 55.60 & 54.62 & 50.71 & 48.02 & 18.23 \\
& \cellcolor{blue!6}{\textbf{CoPE}} & \cellcolor{blue!6}{\textbf{58.11}} & \cellcolor{blue!6}{\textbf{59.60}} & \cellcolor{blue!6}{\textbf{61.23}} & \cellcolor{blue!6}{\textbf{58.68}} & \cellcolor{blue!6}{\textbf{52.72}} & \cellcolor{blue!6}{\textbf{28.48}} \\

\bottomrule[1.1pt]
\end{tabular}
\renewcommand{\arraystretch}{1.0}
\end{table*}

In this section, we evaluate CoPE across various benchmarks to answer the following questions: \textbf{(1)} Does CoPE consistently outperform RoPE and the hard clipping strategy on real-world long context tasks? \textbf{(2)} Are synthetic benchmarks reliable proxies for real-world performance? \textbf{(3)} Can CoPE retain performance on short context benchmarks that assess general model capabilities? \textbf{(4)} How does the choice of clipping onset affect performance?

\subsection{Experimental Setups}

\textbf{Evaluation Benchmarks.} For long context evaluation, we primarily utilize the HELMET benchmark \cite{yen2025helmet}, which improves upon purely synthetic benchmarks (e.g., RULER \cite{hsieh2024ruler}) and benchmarks with limited real-world tasks (e.g., InfiniteBench \cite{zhang-etal-2024-bench}), providing a more robust and realistic assessment. HELMET includes both synthetic recall and a diverse set of real-world tasks, including retrieval-augmented generation (RAG), many-shot in-context learning (ICL), long-document QA, and summarization. We also report results on synthetic tasks from RULER and InfiniteBench. For standard short context benchmarks, we adopt MMLU \cite{hendryckstest2021}, MMLU-Pro \cite{wang2024mmlu}, GPQA \cite{rein2024gpqa}, BIG-Bench Hard \cite{suzgun2022challenging}, and GSM8K \cite{cobbe2021gsm8k}. For more detailed benchmark descriptions, please refer to Appendix~\ref{app:benchmark}.

\textbf{Long Context Training Stage.} We employ Llama-3-8B \cite{grattafiori2024llama} as the backbone model, which is pre-trained with an 8k context window. We extend the models' context length to 64k via continued pre-training on ProLong data ($20$B tokens) \cite{gao-etal-2025-train}, followed by SFT on UltraChat ($1$B tokens) \cite{ding2023enhancing}. Following Qwen3 and ProLong \cite{yang2025qwen3, gao-etal-2025-train}, we increase the base frequency from $5\times10^{5}$ to $1\times10^{7}$
using the ABF technique \cite{xiong-etal-2024-effective}.

\textbf{Baselines.} We compare CoPE with the widely-used RoPE \cite{su2024roformer} and a hard clipping strategy that directly sets some low frequencies to zero \cite{barbero2025round}.

\textbf{Implementation Details.} For both continued pre-training and SFT, we adopt a batch size of $256$ ($16$M tokens) and the AdamW optimizer \cite{loshchilov2017decoupled} with a weight decay of $0.1$ and $(\beta_1, \beta_2) = (0.9, 0.95)$. Both stages are trained for one epoch, differing only in their learning rate schedules. Specifically, continued pre-training uses an initial learning rate of $1\times10^{-5}$ with a $10\%$ warmup and cosine decay to $1\times10^{-6}$, while SFT uses an initial learning rate of $2\times10^{-5}$ with a $5\%$ warmup and cosine decay to $2\times10^{-6}$. The clipping onset is set to $44$ ($64$ frequencies in total). For evaluations beyond 64k, we leverage YaRN \cite{peng2024yarn} with a scaling factor of 4. The training process takes approximately $1996$ and $48$ GPU hours on machines equipped with H100-80GB GPUs, respectively.

\begin{table*}[t!]
\centering
\renewcommand{\arraystretch}{1.13}
\small
\caption{\textbf{Performance comparison on synthetic tasks} sampled from InfiniteBench and RULER, which provide limited insights into real-world performance.}
\label{tab:cope_long_context}
\begin{tabular}{>{\centering\arraybackslash}p{2.7cm}|>{\raggedright\arraybackslash}p{1.3cm}|*{6}{>{\centering\arraybackslash}p{1.3cm}}}
\toprule[1pt]

\multirow{3}{*}{\textbf{Task}} 
& \multirow{2}{*}{\textbf{Method}} 
& \multicolumn{6}{c}{\textbf{Context Length}} \\[4pt]
& & \textbf{8k} & \textbf{16k} & \textbf{32k} & \textbf{64k} & \textbf{128k} & \textbf{256k} \\
\midrule

\multirow{3}{*}{\textbf{RULER NIAH}}
& RoPE & 100.0 & 100.0 & 98.67 & 99.67 & 91.00 & 60.50 \\
& HardClip & 100.0 & 100.0 & 97.00 & 99.33 & 89.33 & 57.00 \\
& \cellcolor{blue!6}{\textbf{CoPE}} & \cellcolor{blue!6}{\textbf{100.0}} & \cellcolor{blue!6}{\textbf{100.0}} & \cellcolor{blue!6}{\textbf{98.67}} & \cellcolor{blue!6}{\textbf{99.67}} & \cellcolor{blue!6}{\textbf{94.00}} & \cellcolor{blue!6}{\textbf{78.50}}\\
\midrule

\multirow{3}{*}{\textbf{RULER MK}}
& RoPE & 100.0 & 100.0 & 99.00 & 96.00 & 54.33 & 27.00 \\
& HardClip & 100.0 & 100.0 & 98.00 & 89.33 & 63.33 & 17.67 \\
& \cellcolor{blue!6}{\textbf{CoPE}} & \cellcolor{blue!6}{\textbf{100.0}} & \cellcolor{blue!6}{\textbf{100.0}} & \cellcolor{blue!6}{98.67} & \cellcolor{blue!6}{\textbf{99.67}} & \cellcolor{blue!6}{59.67} & \cellcolor{blue!6}{\textbf{32.33}}\\
\midrule

\multirow{3}{*}{\textbf{InfBench KV}}
& RoPE & 6.20 & 12.80 & 24.60 & 24.20 & 16.40 & 16.40 \\
& HardClip & 6.20 & 12.80 & 24.80 & 11.20 & 21.40 & 21.40 \\
& \cellcolor{blue!6}{\textbf{CoPE}} & \cellcolor{blue!6}{\textbf{6.20}} & \cellcolor{blue!6}{\textbf{12.80}} & \cellcolor{blue!6}{\textbf{25.40}} & \cellcolor{blue!6}{\textbf{31.40}} & \cellcolor{blue!6}{19.40} & \cellcolor{blue!6}{19.40}\\
\midrule

\multirow{3}{*}{\textbf{InfBench Math Find}}
& RoPE & 37.14 & 35.00 & 36.86 & 33.42 & 35.14 & 35.14 \\
& HardClip & 34.86 & 35.67 & 35.71 & 26.29 & 34.57 & 34.57 \\
& \cellcolor{blue!6}{\textbf{CoPE}} & \cellcolor{blue!6}{35.43} & \cellcolor{blue!6}{\textbf{35.71}} & \cellcolor{blue!6}{36.00} & \cellcolor{blue!6}{\textbf{34.00}} & \cellcolor{blue!6}{\textbf{35.14}} & \cellcolor{blue!6}{\textbf{35.14}}\\
% \midrule

% \multirow{3}{*}{\textbf{InfBench ZH QA}}
% & RoPE & 10.60 & 11.62 & 13.28 & 11.91 & 12.24 & 10.54 \\
% & HardClip &  &  &  &  &  &  \\
% & \cellcolor{blue!6}{\textbf{CoPE}} & \cellcolor{blue!6}{\textbf{14.93}} & \cellcolor{blue!6}{\textbf{18.43}} & \cellcolor{blue!6}{\textbf{21.01}} & \cellcolor{blue!6}{\textbf{13.83}} & \cellcolor{blue!6}{\textbf{19.05}} & \cellcolor{blue!6}{\textbf{12.27}}\\

\bottomrule[1pt]
\end{tabular}
\label{tab:synthetic}
\renewcommand{\arraystretch}{1.0}
\end{table*}

\vspace{0.5em}
\subsection{Main Results}

We evaluate CoPE across a diverse set of tasks, covering synthetic recall, RAG, ICL, QA, and summarization. The results are detailed in Table~\ref{tab:main_results}.

\textbf{Performance on HELMET.} As shown in Table~\ref{tab:main_results}, CoPE consistently outperforms RoPE and HardClip across nearly all tasks and context lengths. Within the training range (64k), CoPE yields an average improvement of $10.84\%$ over RoPE, indicating that soft clipping does not compromise in-distribution performance. When extrapolated to 256k context, CoPE achieves approximately $2\times$ the performance of RoPE, demonstrating superior length generalization ability. In contrast, although the hard clipping strategy slightly improves performance at extreme context lengths (128k-256k), it exhibits noticeable degradation within the training range (8k-64k). This behavior empirically validates our theoretical analysis in Theorem~\ref{theorem:hard_clip}, which highlights that abrupt hard truncation would cause spectral leakage and introduce spurious correlations. Together, these results establish CoPE as a plug-and-play enhancement for vanilla RoPE in long context LLMs, effectively mitigating OOD outliers, refining long-range semantic signals, and preventing spectral leakage induced by hard clipping.

\textbf{Scalable Performance Gain of CoPE.} Beyond higher \textit{absolute} performance, CoPE exhibits performance gains that scale favorably with increasing context length. In particular, the average performance gain is roughly $4.54\%$ at shorter contexts (8-16k), increases to $10.39\%$ within the training range (32k–64k), and further scales to $58.61\%$ under long-context extrapolation (128k–256k). This trend shows that soft clipping effectively suppresses unstable low-frequency behaviors that become pronounced as the context grows.

% \subsection{Limitations of Synthetic Tasks}
% While synthetic recall tasks have been adopted for long context evaluation, we find that they provide limited insight into real-world performance. As shown in Table~\ref{tab:main_results} and Table~\ref{tab:synthetic}, synthetic recall scores either saturate within the training range or exhibit unstable fluctuations under length extrapolation, making them poor proxies for realistic task performance. This observation is consistent with prior work \cite{zhang-etal-2024-bench, gao-etal-2025-train, shang2025longrope} and motivates our use of the more realistic HELMET benchmark.

\vspace{0.5em}

\subsection{Limitations of Synthetic Tasks}
\label{sec:synthetic}

While synthetic recall tasks are widely adopted for long context evaluation, we find that they provide limited insights into real-world performance, as shown in Table~\ref{tab:synthetic}.

\textbf{Saturation Issue.} Many synthetic tasks quickly saturate within the training range, making them ineffective for distinguishing model capabilities. 
For example, RULER-NIAH and RULER-MK achieve near-perfect accuracy for all methods at 8k-64k context lengths, despite significant performance gaps on real-world tasks, as shown in Table~\ref{tab:main_results}. 

% Consequently, such tasks fail to reveal meaningful improvements in modeling capacity.

\textbf{Limited Discriminative Power.} Some synthetic tasks exhibit hardly distinguishable performance across methods by design. For example, on InfiniteBench KV, all methods achieve nearly identical accuracy at 8k-32k contexts, making the task uninformative for comparing model capabilities.

\textbf{Length Invariance.} Furthermore, some other synthetic tasks demonstrate insensitivity to context length. For instance, InfiniteBench Math Find is a variant of multiple numerical lookup and exhibits only minor performance differences across context lengths, i.e., maintaining $\sim$35\% accuracy from 8k to 256k context for all methods.

Overall, synthetic tasks either saturate early or fail to capture meaningful distinctions between models, rendering them poor proxies for real-world long context performance. This observation aligns with prior findings~\cite{gao-etal-2025-train} and motivates our adoption of the HELMET benchmark.

\begin{table}[t!]
\vspace{0.25em}
  \caption{{\textbf{Performance on standard benchmarks} that measure general model capabilities.} Despite clipped low frequencies, CoPE preserves performance and even yields slight gains.}
  \label{tab:short_context}
    \centering
    \resizebox{\linewidth}{!}{%
        \begin{tabular}{lccccc}
          \toprule[1pt]
          Method  & MMLU & MMLU-Pro & GPQA  & BBH & GSM8K \\
          \midrule
          RoPE  & 62.22 & 33.52 & 28.75 & 64.47 & 52.38 \\
          HardClip  & 62.35 & 33.95 & 29.67 & 64.48 & 52.59 \\
          \rowcolor{blue!6}
          \midrule
          \textbf{CoPE}   & \textbf{62.37} & \textbf{34.05} & 29.31 & \textbf{64.51} & 52.46 \\
          \bottomrule[1pt]
        \end{tabular}
    }
\vspace{-1.5em}
\end{table}

\subsection{Results on Standard Short Context Benchmarks}

To verify that CoPE’s soft clipping strategy does not compromise general model capabilities, we evaluate it on a suite of standard short context benchmarks. As shown in Table~\ref{tab:short_context}, CoPE preserves performance and even yields slight gains on all benchmarks, which serve as proxies for broad reasoning and knowledge. The fact that CoPE does not trade off these capabilities indicates that soft clipping primarily suppresses \emph{the suboptimal behavior of low-frequency components}, rather than erasing semantically useful signal. These results, together with CoPE’s consistent gains across context lengths (Table~\ref{tab:main_results}), support our central claim: \emph{soft} clipping is a drop-in enhancement of RoPE that delivers consistent performance gains across tasks and context lengths.

\begin{table}[t]
  \caption{\textbf{Ablation results on HELMET.} While CoPE remains robust to the choice of clipping onset, we find that preserving some stable low frequencies generally yields better performance. CoPE-29 denotes softly clipping the last 29 frequencies, whose periods are longer than the pre-training context window.}
  \label{tab:ablation}
    \centering
    \resizebox{\linewidth}{!}{%
        \begin{tabular}{lcccccc}
          \toprule[1pt]
          Method  & 8k & 16k & 32k & 64k & 128k & 256k \\
          \midrule
          \rowcolor{blue!6}
          \textbf{CoPE} & \textbf{58.11} & \textbf{59.60} & \textbf{61.23} & \textbf{58.68} & \textbf{52.72} & \textbf{28.48} \\
         \midrule
          RoPE  & 55.74  & 56.86  & 55.70  & 52.94  & 44.29  & 14.37 \\
          CoPE-29  & 55.92 & 57.15 & 58.02 & 56.28 & 49.71 & 21.76 \\
          CoPE-34  & 57.09 & 59.46 & 59.55 & 57.54 & 49.33 & 19.07 \\
          \bottomrule[1pt]
        \end{tabular}
    }
\vspace{-1em}
\end{table}

\subsection{Ablation Study}

To understand how the choice of clipping onset impacts performance, we conduct an ablation study by varying the number of frequencies that are softly clipped in CoPE. The results are summarized in Table~\ref{tab:ablation}. 

Specifically, we consider two variants, CoPE-29 and CoPE-34, which softly clip a larger portion of the low-frequency components compared to the default configuration (CoPE-20). In CoPE-29, all frequencies whose periods exceed the pre-training context window are clipped, while CoPE-34 further removes part of the moderately low-frequency band.

According to Table~\ref{tab:ablation}, we observe that: (1) CoPE remains robust to the choice of clipping onset, with all variants outperforming vanilla RoPE across different context lengths. (2) The default CoPE configuration, which clips $\sim75\%$ of the low frequencies, consistently yields the best performance, indicating that low-frequency suppression, while effective, should avoid being overly aggressive.

\section{Related Work}
RoPE is widely adopted in modern LLMs and is deeply coupled with their length generalization ability. To enable context extension, prior work has proposed various modifications to RoPE. In this work, we highlight that their underlying guiding principles can be generally categorized into two classes: \emph{OOD mitigation} and \emph{semantic modeling}.

\textbf{RoPE OOD Mitigation.} The low-frequency components in RoPE possess periods longer than the pre-training context window, which will lead to severe OOD issues during extrapolation. To mitigate this, a line of work has investigated different methods to scale RoPE frequencies so that extended contexts are mapped back to the original training range, including PI \cite{chen2023extending}, NTK \cite{NTK}, YaRN \cite{peng2024yarn}, and LongRoPE \cite{ding2024longrope, shang2025longrope}. As discussed in Section~\ref{sec:rope_ood}, these methods differ primarily in their choice of per-frequency scaling factors, and the key technique is to interpolate low frequencies while minimizing the impact on high frequencies, which have completed multiple cycles during pre-training.

\textbf{RoPE Semantic Modeling.} Meanwhile, another line of work has investigated how the semantic information is carried within RoPE. As discussed in Section~\ref{sec:rope_semantic}, Men et al. \yrcite{men2024base} observe that besides the general decay of activations, RoPE also introduces an undesirable decay property: the ability to attend more to semantically similar tokens than random ones decays as the relative distance increases, which we refer to as \emph{long-term decay of semantic attention}. To alleviate this decay, they propose a higher base frequency strategy, which is also introduced in the ABF technique \cite{xiong-etal-2024-effective}. More recently, several studies analyze the attention patterns within different RoPE frequencies, revealing that low-frequency components primarily carry semantic information, as they are the most invariant to token relative distance \cite{barbero2025round, jin2025massive}.

In our work, we unify these seemingly diverging objectives and argue that they stem from the same issue: \emph{the suboptimal behavior
of low-frequency components in the extrapolation regime.} This is inspired by the fact that the low-frequency components are responsible for OOD extrapolation, while simultaneously serving as an unreliable semantic channel whose discriminative power decays with increasing relative distance. Given this insight, we propose a minimalist and principled enhancement, termed CoPE, which softly clips the low-frequency components of RoPE to suppress OOD outliers and refine long-range semantic signals. Importantly, softly clipping prevents spectral leakage induced by hard frequency truncation \cite{barbero2025round}, which can introduce ringing artifacts and spurious correlations.

\section{Conclusion}
In this paper, we present a unified perspective on long context adaptations of RoPE. We first highlight that existing methods can be categorized into two paradigms: OOD mitigation and semantic modeling. Then, we point out that these two seemingly distinct objectives originate from the same issue: \emph{the suboptimal behavior of low-frequency components in the extrapolation regime.} Motivated by this insight, we introduce CoPE, a plug-and-play enhancement for RoPE that softly clips the low-frequency components. CoPE not only suppresses OOD outliers and refines long-range semantic signals, but also avoids spectral leakage induced by hard frequency truncation. Extensive experiments on a diverse set of real-world tasks demonstrate that CoPE consistently outperforms RoPE and the hard clipping strategy across context lengths of up to 256k, confirming its effectiveness and moving beyond prior perplexity-based metrics, synthetic recall benchmarks, and short context evaluation.

\section*{Impact Statement}

This paper presents work whose goal is to advance the field of Machine
Learning. There are many potential societal consequences of our work, none
which we feel must be specifically highlighted here.

\nocite{li2025hope}

\bibliography{example_paper}
\bibliographystyle{icml2026}

%%%%%%%%%%%%%%%%%%%%%%%%%%%%%%%%%%%%%%%%%%%%%%%%%%%%%%%%%%%%%%%%%%%%%%%%%%%%%%%
%%%%%%%%%%%%%%%%%%%%%%%%%%%%%%%%%%%%%%%%%%%%%%%%%%%%%%%%%%%%%%%%%%%%%%%%%%%%%%%
% APPENDIX
%%%%%%%%%%%%%%%%%%%%%%%%%%%%%%%%%%%%%%%%%%%%%%%%%%%%%%%%%%%%%%%%%%%%%%%%%%%%%%%
%%%%%%%%%%%%%%%%%%%%%%%%%%%%%%%%%%%%%%%%%%%%%%%%%%%%%%%%%%%%%%%%%%%%%%%%%%%%%%%
\newpage
\appendix
\onecolumn
\section{Proofs}
In this section, we provide detailed proofs for the theoretical statements presented in this paper.

\subsection{Long-term Decay of Semantic Attention}
\label{proof:semantic_decay}

As discussed in Theorem~\ref{theorem:semantic_attention}, RoPE secretly induces a long-term decay of semantic attention, where the ability to attend more to semantically similar tokens than random ones decays as the relative distance increases. Here, we provide the derivation used in Equation~\ref{eq:semantic_attn}.

\begin{tcolorbox}[
  colback=gray!7!white,  
  colframe=black,  
  arc=0.8mm,                 
  boxrule=1pt,         
  left=2mm, right=2mm, top=1mm, bottom=1mm,
  enhanced,
  width=\linewidth,
]
\textbf{Theorem~\ref{theorem:semantic_attention}} \textbf{(}Long-term Decay of Semantic Attention\textbf{).} \textit{Assume query $\mathbf{q} \in \mathbb{R}^d$ and key $\mathbf{k} \in \mathbb{R}^d$ have distance $\Delta t$ and \text{i.i.d.} components with standard deviation $\sigma$. Let $\mathbf{k}^{\prime}=\mathbf{q}+\mathbf{\epsilon}$ denote a similar key to the query, where $\mathbf{\epsilon}$ is a zero-mean perturbation. Then, we have:}
\begin{equation}
    \mathbb{E}_{\mathbf{q},\mathbf{k}, \epsilon}[\mathbf{q}^{\top}\mathbf{R}_{\Delta t}\mathbf{k}^{\prime}-\mathbf{q}^{\top}\mathbf{R}_{\Delta t}\mathbf{k}] = {2\sigma^2} \sum_{i=0}^{d/2-1} \cos(\Delta t \theta_i),
\end{equation}
\textit{where $\sum_{i=0}^{d/2-1} \cos(\Delta t \theta_i)$ decays as $\Delta t$ increases.}
\end{tcolorbox}

\begin{proof}
\begin{equation}
\begin{aligned}
    \mathbb{E}_{\mathbf{q},\mathbf{k},\epsilon}[\mathbf{q}^{\top}\mathbf{R}_{\Delta t}\mathbf{k}'-\mathbf{q}^{\top}\mathbf{R}_{\Delta t}\mathbf{k}] &=\mathbb{E}_{\mathbf{q},\mathbf{k},\epsilon}[\mathbf{q}^{\top}\mathbf{R}_{\Delta t}\mathbf{(q+\epsilon)}-\mathbf{q}^{\top}\mathbf{R}_{\Delta t}\mathbf{k}] 
    \\[4pt]
    &=\mathbb{E}_{\mathbf{q,\epsilon}}[\mathbf{q}^{\top}\mathbf{R}_{\Delta t}(\mathbf{q}+\epsilon)] -\mathbb{E}_{\mathbf{q},\mathbf{k}}[\mathbf{q}^{\top}\mathbf{R}_{\Delta t}\mathbf{k}] 
    \\[4pt]
    &=\mathbb{E}_{\mathbf{q}}[\mathbf{q}^{\top}\mathbf{R}_{\Delta t}\mathbf{q}] + \mathbb{E}_{\mathbf{q,\epsilon}}[\mathbf{q}^{\top}\mathbf{R}_{\Delta t}\epsilon] -\mathbb{E}_{\mathbf{q},\mathbf{k}}[\mathbf{q}^{\top}\mathbf{R}_{\Delta t}\mathbf{k}]
    \\[4pt]
    &=\mathbb{E}_{\mathbf{q}}[\mathbf{q}^{\top}\mathbf{R}_{\Delta t}\mathbf{q}] -\mu^2 \mathbf{1}^{\top} \mathbf{R}_{\Delta t} \mathbf{1}
    \\[4pt]
    &=\mathbb{E}_{\mathbf{q}}[\sum_{i=0}^{d/2-1}(\mathbf{q}_{2i}^2+\mathbf{q}_{2i+1}^2)\cos(\Delta t\theta_i)]- \sum_{i=0}^{d/2-1}2\mu^2\cos(\Delta t\theta_i)
    \\[4pt]
    &=\sum_{i=0}^{d/2-1}2(\mu^2+\sigma^2)\cos(\Delta t\theta_i) - \sum_{i=0}^{d/2-1}2\mu^2\cos(\Delta t\theta_i)
    \\[4pt]
    &=2\sigma^2\sum_{i=0}^{d/2-1}\cos(\Delta t\theta_i),
\end{aligned}
\end{equation}

where $\mu$ denotes the mean of the i.i.d.\ components in $\mathbf q$ and $\mathbf k$. The term $\sum_{i=0}^{d/2-1}\cos(\Delta t\theta_i)$ is oscillatory and thus not monotonic in $\Delta t$, but exhibits a general decay as $\Delta t$ increases, as shown in Figure~\ref{fig:semantic_decay}.
\end{proof}

\subsection{Spectral Leakage from Hard Clipping}
\label{proof:hard_clip}

\begin{tcolorbox}[
  colback=gray!7!white,  
  colframe=black,  
  arc=0.8mm,                 
  boxrule=1pt,         
  left=2mm, right=2mm, top=1mm, bottom=1mm,
  enhanced,
  width=\linewidth,
]
\textbf{Theorem~\ref{theorem:hard_clip}} \textbf{(}Spectral Leakage from Hard Clipping\textbf{).} \textit{Let $A(\tau)$ be the continuous attention score. Hard high-pass filter at cutoff $\theta_c$ yields a $\tilde{A}(\tau) = A(\tau) + E(\tau)$, where the error term is:}
\begin{equation}
E(\tau) = - A(\tau) * \left( \frac{\theta_c}{\pi} \text{sinc}\left(\frac{\theta_c \tau}{\pi}\right) \right).
\end{equation}
\textit{The slow $O(1/\tau)$ decay of the sinc kernel introduces Gibbs oscillations, disrupting the general decay of $A(\tau)$ and inducing spurious long-range correlations.}
\end{tcolorbox}

\begin{proof}
Let $\mathcal{F}$ denote the Fourier transform and $\mathcal{F}^{-1}$ the inverse Fourier transform. We define the operation of hard clipping at cutoff frequency $\theta_c$ as applying an ideal high-pass filter, $H_{\text{high}}(\omega)$, in the frequency domain. This filter can be expressed as the complement of an ideal low-pass filter (rectangular window), $H_{\text{low}}(\omega)$:
\begin{equation}
    H_{\text{high}}(\omega) = 1 - H_{\text{low}}(\omega), \quad \text{where} \quad H_{\text{low}}(\omega) = \mathbb{I}(|\omega| \le \theta_c).
\end{equation}
Let $\hat{A}(\omega) = \mathcal{F}[A(\tau)]$ be the spectrum of the continuous attention score. The spectrum of the filtered signal, $\hat{\tilde{A}}(\omega)$, is given by the element-wise product:
\begin{equation}
\begin{aligned}
    \hat{\tilde{A}}(\omega) &= \hat{A}(\omega) \cdot H_{\text{high}}(\omega) \\
    &= \hat{A}(\omega) \cdot (1 - H_{\text{low}}(\omega)) \\
    &= \hat{A}(\omega) - \hat{A}(\omega) \cdot H_{\text{low}}(\omega).
\end{aligned}
\end{equation}
By the Convolution Theorem, multiplication in the frequency domain corresponds to convolution in the time domain. Applying the inverse Fourier transform $\mathcal{F}^{-1}$ to both sides yields:
\begin{equation}
    \tilde{A}(\tau) = A(\tau) - \left( A(\tau) * \mathcal{F}^{-1}[H_{\text{low}}(\omega)](\tau) \right).
\end{equation}
The inverse Fourier transform of the rectangular function $H_{\text{low}}(\omega)$ with cutoff $\theta_c$ is the normalized sinc function:
\begin{equation}
    \mathcal{F}^{-1}[H_{\text{low}}(\omega)](\tau) = \frac{\theta_c}{\pi} \text{sinc}\left(\frac{\theta_c \tau}{\pi}\right).
\end{equation}
Substituting this kernel back into the time-domain equation, we identify the error term $E(\tau) = \tilde{A}(\tau) - A(\tau)$ as:
\begin{equation}
    E(\tau) = - A(\tau) * \left( \frac{\theta_c}{\pi} \text{sinc}\left(\frac{\theta_c \tau}{\pi}\right) \right).
\end{equation}
This concludes the derivation. The impulse response of the ideal low-pass filter is a sinc function, which decays asymptotically as $O(1/\tau)$. This slow decay manifests as Gibbs oscillations (ringing artifacts) in the time domain, disrupting the general decay of $A(\tau)$ and inducing suprious long-range correlations. This negative effect is also illustrated in Figure~\ref{fig:ring}.
\end{proof}

\section{Further Experimental Details}
In this section, we provide further details of our experiments, including benchmark descriptions, additional results, and a case study.

\subsection{Benchmark Description}
\label{app:benchmark}

In this subsection, we provide detailed descriptions of the long context benchmarks we used in the experiments, including HELMET \cite{yen2025helmet}, RULER \cite{hsieh2024ruler}, and Infinite Bench \cite{zhang-etal-2024-bench}.

\begin{itemize}

\item \textbf{HELMET} is a comprehensive benchmark for evaluating long context LLMs on real-world tasks, improving upon purely synthetic benchmarks (e.g., RULER) and benchmarks with limited real-world tasks (e.g., Infinite Bench). Specifically, HELMET comprises summarization, long-document QA, many-shot in-context learning (ICL), synthetic recall, retrieval-augmented generation (RAG), generation with citations, and passage re-ranking. Following ProLong \cite{gao-etal-2025-train}, we select the five most representative tasks for evaluation.

\item \textbf{RULER} is a purely synthetic benchmark for long context evaluation, which expands upon the vanilla needle-in-a-haystack (NIAH) test to incorporate variations with diverse types and quantities of needles, resulting in a total of 13 synthetic tasks. However, as shown in recent work \cite{yen2025helmet, gao-etal-2025-train, zhang-etal-2024-bench, shang2025longrope} and our Section~\ref{sec:synthetic}, synthetic tasks either saturate quickly within the training range or provide limited signals for real-world performance, rendering them poor proxies for long context capabilities.

\item \textbf{Infinite Bench} is a benchmark designed to evaluate LLMs on extremely long-context understanding, consisting of both synthetic and real-world tasks with an average length of $\sim 200$k. Infinite Bench covers domains such as novel understanding, code execution, and mathematical calculation. Nevertheless, as discussed in Section~\ref{sec:synthetic}, we find that some tasks exhibit limited discriminative power among different methods (KV Retrieval) or insensitivity to context length (Math Find), which motivates our use of the more realistic HELMET benchmark.
\end{itemize}

\subsection{Additional Results}

We report the quantitative results of CoPE and RoPE on the RULER benchmark in Table~\ref{tab:ruler}. We observe that, except under extremely long contexts (256k), where CoPE achieves a substantial improvement (up to $+18.0$), most RULER tasks exhibit limited discriminative power between different methods. In contrast, on real-world tasks from the HELMET benchmark, such as RAG, in-context learning, and long-form summarization, CoPE consistently yields significant performance gains, as shown in Table~\ref{tab:main_results}. These results suggest that synthetic recall benchmarks may fail to fully reflect practical long context capabilities, highlighting the importance of evaluating different methods on realistic downstream tasks.

\begin{table}[t]
  \caption{\textbf{Performance comparison on the RULER benchmark.} The results are averaged across 13 tasks.}
  \label{tab:ruler}
    \centering
        \begin{tabular}{lcccccc}
          \toprule[1pt]
          Method  & 8k & 16k & 32k & 64k & 128k & 256k \\
          \midrule
          RoPE  & 80.52 & 82.33  & 82.11  & 76.93  & 61.19  & 28.86 \\
          \rowcolor{blue!6}
          \textbf{CoPE} & \textbf{81.50} & \textbf{82.84} & \textbf{82.75} & 76.71 & \textbf{61.95} & \textbf{46.86} \\
          \bottomrule[1pt]
        \end{tabular}
\vspace{-1em}
\end{table}

\subsection{Case Study}

Table~\ref{tab:case_study} presents some long-document QA examples with RoPE and CoPE. We observe that RoPE exhibits repetitive and less informative responses under long-context settings, often missing fine-grained details, whereas CoPE produces more coherent and detail-preserving answers.

\begin{table}[h]
\caption{\textbf{Long-document QA examples with RoPE and CoPE.}}
  \label{tab:case_study}
    \centering
        \begin{tabular}{p{7cm} | p{7cm}}
          \toprule[1pt]
          \textbf{Q:} Viola wonders about the fate of which animals in Central Park?
          \par
          \vspace{0.5em}
          \textbf{A:} The ducks in the lagoon.
          \par
          \vspace{0.5em}
          \textbf{Input Length:} 256k
          & 
          \textbf{CoPE:} Viola wonders about the fate of the \textcolor[HTML]{2E7D32}{ducks in the lagoon} in Central Park.
           \par
          \vspace{0.5em}
          \textbf{RoPE:} Viola wonders about the fate of the \textcolor[HTML]{C62828}{ducks [missing key details]} in Central Park. \textcolor[HTML]{C62828}{I'm not sure if he ever found out what happened to the ducks in the winter, but he was worried about it. [non-informative]}

          \\
          \midrule
           \textbf{Q:} What does Raul call his speakeasy?
          \par
          \vspace{0.5em}
          \textbf{A:} The Zebra Room.
        \par
          \vspace{0.5em}
          \textbf{Input Length:} 256k
          & 
          \textbf{CoPE:} \textcolor[HTML]{2E7D32}{The Zebra Room.}
           \par
          \vspace{0.5em}
          \textbf{RoPE:} \textcolor[HTML]{2E7D32}{The Zebra Room.} \textcolor[HTML]{C62828}{"Why did you have to paint a church?" "Why not?" "Because I promised to fix it up." "Why did you promise to fix it?" [non-informative]}
          \\
          \midrule
           \textbf{Q:} In Bellamy's vision of a perfect society, what is needed for artistic or literary success?
          \par
          \vspace{0.5em}
          \textbf{A:} Talent.
          \par
          \vspace{0.5em}
          \textbf{Input Length:} 256k
          & 
          \textbf{CoPE:} The ability to produce work that is of high quality and that is appreciated by the people. The nation provides opportunities for education and development of \textcolor[HTML]{2E7D32}{talent}, and the people vote on the acceptance of work in these...
           \par
          \vspace{0.5em}
          \textbf{RoPE:} The ability to produce work that is of high quality and that is appreciated by the people. \textcolor[HTML]{C62828}{The people are the sole judges of artistic and literary merit... [missing key details]}
          \\
          \bottomrule[1pt]
        \end{tabular}
\vspace{-1em}
\end{table}

% You can have as much text here as you want. The main body must be at most $8$
% pages long. For the final version, one more page can be added. If you want, you
% can use an appendix like this one.

% The $\mathtt{\backslash onecolumn}$ command above can be kept in place if you
% prefer a one-column appendix, or can be removed if you prefer a two-column
% appendix.  Apart from this possible change, the style (font size, spacing,
% margins, page numbering, etc.) should be kept the same as the main body.
%%%%%%%%%%%%%%%%%%%%%%%%%%%%%%%%%%%%%%%%%%%%%%%%%%%%%%%%%%%%%%%%%%%%%%%%%%%%%%%
%%%%%%%%%%%%%%%%%%%%%%%%%%%%%%%%%%%%%%%%%%%%%%%%%%%%%%%%%%%%%%%%%%%%%%%%%%%%%%%

\end{document}